%% file: icml2026.tex
\documentclass{article}

\usepackage{microtype}
\usepackage{graphicx}
\usepackage{subcaption}

\usepackage{booktabs} % for professional tables
\usepackage{bm}

\usepackage[preprint]{icml2026} 
\usepackage{hyperref}

\usepackage{icml2026}

\usepackage{amsmath}
\usepackage{amssymb}
\usepackage{mathtools}
\usepackage{amsthm}
\usepackage{siunitx} 
\usepackage{wrapfig}

\input{math_commands.tex}

\usepackage[capitalize,noabbrev]{cleveref}

\theoremstyle{plain}
\newtheorem{theorem}{Theorem}[section]

\theoremstyle{definition}

\theoremstyle{remark}

\newcommand{\ours}{IDDM}

\usepackage[textsize=tiny]{todonotes}

\icmltitlerunning{Interpolating Discrete Diffusion Models with Controllable Resampling}

\begin{document}

\twocolumn[
\icmltitle{	
Interpolating Discrete Diffusion Models with Controllable Resampling}



\icmlsetsymbol{equal}{*}

\begin{icmlauthorlist}
\icmlauthor{Marcel Kollovieh}{equal,yyy,comp,sch}
\icmlauthor{Sirine Ayadi}{equal,yyy,comp}
\icmlauthor{Stephan Günnemann}{yyy,comp,sch}
\end{icmlauthorlist}

\icmlaffiliation{yyy}{School of Computation, Information and Technology, Technical University of Munich}
\icmlaffiliation{comp}{Munich Data Science Institute}
\icmlaffiliation{sch}{Munich Center for Machine Learning}

\icmlcorrespondingauthor{Marcel Kollovieh}{m.kollovieh@tum.de}

\icmlkeywords{Machine Learning, ICML}

\vskip 0.3in
]



\printAffiliationsAndNotice{\icmlEqualContribution} % otherwise use the standard text.
\crefname{equation}{equation}{equations}

\input{sections/0_abstract}
\input{sections/1_introduction}
\input{sections/2_background}
\input{sections/3_methodology}

\input{sections/4_experiments}
\input{sections/5_related_work}

\input{sections/6_conclusion}
%\clearpage
%\input{includes/mixed}
\bibliography{example_paper}
\bibliographystyle{icml2026}

%%%%%%%%%%%%%%%%%%%%%%%%%%%%%%%%%%%%%%%%%%%%%%%%%%%%%%%%%%%%%%%%%%%%%%%%%%%%%%%
%%%%%%%%%%%%%%%%%%%%%%%%%%%%%%%%%%%%%%%%%%%%%%%%%%%%%%%%%%%%%%%%%%%%%%%%%%%%%%%
% APPENDIX
%%%%%%%%%%%%%%%%%%%%%%%%%%%%%%%%%%%%%%%%%%%%%%%%%%%%%%%%%%%%%%%%%%%%%%%%%%%%%%%
%%%%%%%%%%%%%%%%%%%%%%%%%%%%%%%%%%%%%%%%%%%%%%%%%%%%%%%%%%%%%%%%%%%%%%%%%%%%%%%
\newpage
\appendix
\onecolumn
\input{sections/7_appendix}
%%%%%%%%%%%%%%%%%%%%%%%%%%%%%%%%%%%%%%%%%%%%%%%%%%%%%%%%%%%%%%%%%%%%%%%%%%%%%%%
%%%%%%%%%%%%%%%%%%%%%%%%%%%%%%%%%%%%%%%%%%%%%%%%%%%%%%%%%%%%%%%%%%%%%%%%%%%%%%%

\end{document}

%% file: math_commands.tex
\def\eqref#1{equation~\ref{#1}}
\def\1{\bm{1}}

\DeclareMathAlphabet{\mathsfit}{\encodingdefault}{\sfdefault}{m}{sl}
\SetMathAlphabet{\mathsfit}{bold}{\encodingdefault}{\sfdefault}{bx}{n}

\newcommand{\x}{\mathbf{x}}
\newcommand{\q}{\mathbf{q}}
\newcommand{\y}{\mathbf{y}}
\newcommand{\m}{\mathbf{m}}

\newcommand{\ex}{\mathbf{e}_\x}

\DeclareMathOperator{\Cat}{Cat}

\newcommand{\e}{\mathbf{e}}

%% file: sections/0_abstract.tex
\begin{abstract}
Discrete diffusion models form a powerful class of generative models across diverse domains, including text and graphs. However, existing approaches face fundamental limitations. Masked diffusion models suffer from irreversible errors due to early unmasking, while uniform diffusion models, despite enabling self-correction, often yield low-quality samples due to their strong reliance on intermediate latent states.
We introduce \textit{\ours}, an Interpolating Discrete Diffusion Model, that improves diffusion by reducing dependence on intermediate latent states. Central to \textit{\ours} is a controllable resampling mechanism that partially resets probability mass to the marginal distribution, mitigating error accumulation and enabling more effective token corrections. \textit{\ours} specifies a generative process whose transitions interpolate between staying at the current state, resampling from a prior, and flipping toward the target state, while enforcing marginal consistency and fully decoupling training from inference.
We benchmark our model against state-of-the-art discrete diffusion models across molecular graph generation as well as text generation tasks, demonstrating competitive performance.
\end{abstract}

%% file: sections/1_introduction.tex
\section{Introduction}
Discrete diffusion models~\citep{austin2021structured} have emerged as a powerful approach for generating structures across domains such as proteins~\citep{gruver2023protein,nisonoff2024unlocking}, DNA sequences~\citep{sarkar2024designing}, and molecules~\citep{vignac2022digress,campbell2022continuous}, and are starting to catch up with autoregressive models in language generation~\citep{lou2023discrete,sahoo2025diffusion,schiff2024simple}. As with their continuous counterparts, the core idea is to define Markov transitions that gradually corrupt the data, and then learn a model to denoise it by parametrizing the reverse process. 
Common forward processes include uniform transitions over the state space or masking, also known as absorbing-state diffusion, in which the state either remains unchanged or transitions to a mask token~\citep{austin2021structured,sahoo2024simple,shi2024simplified}.
%absorbing states, e.g., masking~\citep{sahoo2024simple,shi2024simplified}.

While these frameworks work analogously to the continuous case, they introduce considerable limitations. 
Masked diffusion models, for example, do not support refinement: once a token is unmasked, it remains fixed throughout the entire generation, restricting the model to correct earlier mistakes and leading to error accumulation~\citep{sahoo2024simple,wang2025remasking,shi2024simplified,boget2025simple}. Uniform transitions, on the other hand, allow flips back to the prior, but they can still suffer from error accumulation~\citep{boget2025simple}, and are typically outperformed by absorbing-state diffusion approaches~\citep{austin2021structured,lou2023discrete}.

In this work, we propose \emph{Interpolating Discrete Diffusion Model} (\emph{\ours{}}), a generative model for discrete data that derives transition probabilities from a different perspective. Instead of defining a Markovian forward transition to reach a terminal distribution, we directly define the reverse, i.e., the generative distribution. To this end, our transitions make use of three different actions: \emph{(i)} staying in the current state, \emph{(ii)} transitioning to the goal state, and \emph{(iii)} resampling from the prior distribution, see \cref{fig:graphical_model}.

\input{includes/graphical_model}
By constraining the marginal distribution to match an interpolation from prior to the data distribution, we obtain a family of posterior transitions. The posteriors interpolate between a fully local update without random flips and a state-independent update that discards the current state, enabling us to trade off preservation of the current knowledge and stochasticity.
To demonstrate \ours{}'s capability, we evaluate it on molecule generation and language tasks.
\input{includes/main_figure}
\newpage
\paragraph{Contributions.} Our contributions can be summarized as:
\begin{itemize}
    \item  \textbf{Interpolating reverse dynamics.} We propose the \emph{Interpolating Discrete Diffusion Model (\ours{})}, a simple discrete diffusion model consisting of three different actions, i.e., stay in the current state, transition to the goal state, resample from the prior.
    \item \textbf{Marginal-consistent posterior family.} By enforcing that the marginals adhere to an interpolation between the prior and the data distribution, we derive a family of posterior transitions that trades off preserving the current state against injecting stochasticity.
    \item \textbf{Forward-process interpretation.} We show that the resulting transitions correspond, in general, to a non-Markovian forward process.
    \item \textbf{Empirical evaluation.} We demonstrate that our model achieves competitive performance on well-established benchmarks for molecule generation and language modeling.
\end{itemize}

%% file: includes/graphical_model.tex
\begin{figure}[t]
    \centering
    %\vspace{0.5em}
    \includegraphics[width=\columnwidth]{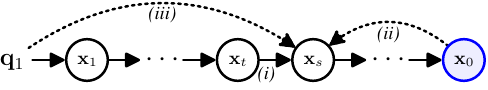}
    \caption{Graphical model of \ours{}. In each step, our model consists of three actions: \emph{(i)} stay in the current state, \emph{(ii)} flip to the goal state, and \emph{(iii)} sample from the prior distribution. The model can be trained without the need to define a forward process.}
    \label{fig:graphical_model}
\end{figure}

%% file: includes/main_figure.tex
\begin{figure*}[t]
    \centering
    \includegraphics[width=\linewidth]{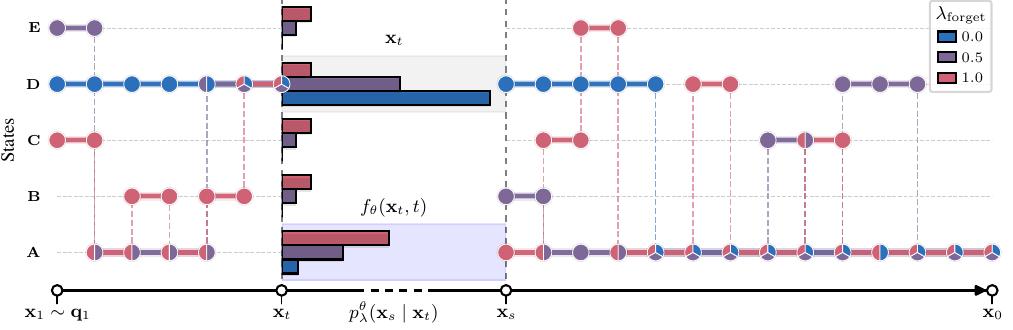}
    \caption{Overview of the generative process of \ours{}. Initial samples are drawn from the prior distribution $\q_1$. In each denoising step $\x_s\mid\x_t$, our model constructs the posterior $p_\lambda^\theta(\x_s\mid\x_t)$. Increasing $\lambda$ increases the stochasticity and state transitions by shifting probability mass from the current to the predicted and random states, as visualized in the bar plot.}
    \label{fig:main_figure}
\end{figure*}

%% file: sections/2_background.tex
\section{Discrete Diffusion Models}
\paragraph{Notations.} For clarity and ease of presentation, we use $\x$ to denote a single discrete random variable taking values in $\{1,\dots,K\}$. 
In sequence modeling, $\x$ represents a single token, while in graph modeling, it can represent a single node or a single edge in the graph. We denote the corresponding one-hot encoding by $\ex \in \Delta^{K-1}$, where $\Delta^{K-1}$ is the simplex over $K$ categories. We denote by $\Cat(\cdot;\y)$ the categorical distribution with probability vector $\y \in \Delta^{K-1}$. 

Diffusion models learn a Markov chain that involves a forward process to perturb the data $\x \sim \q_0(\x)$ to increasingly noisy latent variables $\x_t$, for $t \in [0,1]$, and a trained reverse process. 
The forward process $p(\x_t \mid \x_s)$ is expressed under the Markovian assumption, and the induced marginal for every latent variable $\x_t$ can be defined as an interpolation between the clean data and a noisy prior $\q_1$ \citep{sahoo2024simple}:
\begin{equation}
    p(\x_t \mid \x) \;=\; \Cat(\x_t; (1\;-\;\alpha_t) \q_1 \;+\; \alpha_t \e_\x)\;, 
\end{equation}
where $\alpha_t \in [0,1]$ is a strictly decreasing function in $t$, and unless stated otherwise, $\x_t$ is represented in one-hot form.
The parametrized reverse process $p_\theta(\x_s \mid \x_t, \x), s < t$ is trained to optimize a variational lower bound on the log-likelihood:
\begin{align}
    \log\ & p_{\theta}(\x)
    \geq \notag \\
    &
    \mathbb{E}_{p(\x_{t(1)}, \dots,\x_{t(T)}\mid \x)}\left[
    \log \frac{p_{\theta}(\x,\x_{t(1)}, \dots,\x_{t(T)})}{p(\x_{t(1)}, \dots,\x_{t(T)}\mid \x)}
    \right],
\end{align}   
where $t(i) = i/T$, $T$ is the number of diffusion steps, and the reverse Markov process 
\begin{align}
    &p_{\theta}(\x, \x_{t(1:T)})= \notag\\
    &\qquad p(\x_{t(T)})\,\prod_{i=2}^{T} p_{\theta}(\x_{t(i-1)} \mid\x_{t(i)})\,p_{\theta}(\x \mid \x_{t(1)}), 
\end{align}
gradually denoises the latent states
towards the data distribution. There exist two popular variants for interpolating discrete diffusion models: uniform diffusion models with a stationary distribution $\q_1 = \1 / K$, where $\1$ is a $K-$dimensional column vector of all ones, and masked diffusion models, also known as absorbing state diffusion, where the space is augmented with a special mask token, and the prior distribution is set to $\q_1 = \m$. Marginal distributions over node and edge attributes as noise have proven to be more suitable for graph generation.

Standard discrete diffusion models tightly couple the reverse process to a fixed forward corruption mechanism. The forward process is defined as a Markov chain $p(\x_t \mid \x_s)$, and the reverse transitions can be determined via Bayes' rule:
\[
p(\x_s \mid \x_t, \x) = \frac{p(\x_t \mid \x_s)p(\x_s \mid \x)}{p(\x_t \mid \x)}\;,
\]
where $p(\x_t \mid \x_s) = p(\x_t \mid \x_s, \x)$ follows from the Markov property. As a result, the reverse dynamics are largely determined once the forward kernel and the noise schedule are fixed. This can limit the extent to which reverse updates can be modified independently of the noising process and, consequently, can negatively impact the sample quality. For example, in absorbing-state diffusion, once a token is unmasked, it remains fixed for the remainder of the generative process. This can lead to systematic error accumulation as early mistakes cannot be recovered. Uniform diffusion models, on the other hand, inherently enable self-correction since the latent states $\x_t$ can be continuously updated throughout the reverse process. However, this mechanism can also introduce errors, as correctly predicted tokens can be randomly resampled from the uniform prior. 

%% file: sections/3_methodology.tex
\section{Methodology}
\input{includes/transition_heatmaps}
\input{algorithms/algorithms}
In this section, we propose \emph{Interpolating Discrete Diffusion Model (\ours)}, a simple generative model for discrete data that enables controlled resampling during inference and improves token correction. We provide an overview of \ours\ in \cref{fig:main_figure}.
Our goal is to learn the target distribution $\q_0$. Similar to diffusion models, we define a sequence of latent variables $\x_t$ that interpolate between the data at $t=0$ and a prior distribution at $t=1$. However, unlike diffusion models, we do not explicitly define forward, i.e., noising process $p(\x_{t}\mid \x_{s},\x)$, for $s < t$. Instead, we define the reverse-time generative process directly.

We start by defining our posterior transition probability with three actions: (i) \emph{stay} at the current state,
(ii) \emph{flip} to the goal state $\e_\x$, or (iii) \emph{resample} from the prior $\q_1$. More specifically:
\begin{align}
p(\x_{s}\mid \x_{t},\x) &= \Cat\Big(w_{\mathrm{stay}}\x_t + w_{\mathrm{prior}}\q_1 + w_{\mathrm{flip}}\e_\x \Big), \label{eq:general_posterior}
\end{align}
with $w_{\mathrm{stay}},w_{\mathrm{flip}}, w_{\mathrm{prior}}\geq 0$ and $w_{\mathrm{stay}} + w_{\mathrm{flip}} + w_{\mathrm{prior}}=1$. 
%\begin{align}\label{eq:reverse_lamb_0}
%p(\x_{s}\mid \x_{t},\x)
%&=\Cat\!\Big(\x_s;\gamma_{s\mid t}\,\x_t+(1-\gamma_{s\mid t})\e_\x\Big),
%\end{align}
%with $s<t$, $\gamma_{s\mid t}=\frac{\gamma_s}{\gamma_t}$, and $\x$ being the ground-truth sample. Intuitively, this means that we either stay in our previous state or flip to the ground-truth class. Once $\x_t$ reaches $\e_\x$, all next states get absorbed to $\e_\x$.
\paragraph{Marginal distribution.}
We constrain the marginal distribution to interpolate between the prior distribution and the goal state:
\begin{align}\label{eq:marginal_main}
p(\x_t \mid \x) &= \Cat\!\Big(
\x_t;\, (1-\gamma_t) \q_1 + \gamma_t \e_\x \Big),
\end{align}
with $\gamma_t \in [0,1]$, $\gamma_1=0$ and $\gamma_0=1$.
This choice enables direct sampling of $\x_t$ at arbitrary $t$ without simulating intermediate states. 

\paragraph{Parametrization of transition weights.}
To enforce $p(\x_s\mid \x)=\Cat((1-\gamma_s)\q_1+\gamma_s\e_\x)$ for all $s<t$, our weights must satisfy the following constraints:
\begin{align}
&(1)\, w_{\mathrm{stay}}(1-\gamma_t)+w_{\mathrm{prior}} = 1-\gamma_s, \notag \quad \text{and}
\\
&(2)\, w_{\mathrm{stay}}\gamma_t+w_{\mathrm{flip}} = \gamma_s.
\label{eq:weight_constraints}
\end{align}
We use this to introduce a forgetting parameter
$\lambda\in[0,1]$ that interpolates between a purely local update ($\lambda=0$) and a global reset to the marginal at time $s$ ($\lambda=1$).
%\begin{equation}\label{main:posterior}
%\begin{aligned}
%p_{\lambda}(\x_{s}\mid \x_{t},\x)
%&=\Cat\!\Big(
%\x_s;
%(1-\lambda)\big[\gamma_{s\mid t}\,\x_t+(1-%\gamma_{s\mid t})\e_\x\big] \\
%&\qquad
%+\lambda\big[\gamma_s q_1+(1-%\gamma_s)\e_\x\big]
%\Big),
%\end{aligned}
%\end{equation}

%This construction makes every state an absorbing one in both forward and reverse directions, similar to masking diffusion. However, in our model, this holds for arbitrary prior distributions. 

%\subsection{Resampling from the prior distribution}
%As absorbing states hinder the correction of the previous errors and therefore might accumulate those, we propose a simple extension. Rather than interpolating between the previous state and the goal state, we expand the generative process by including the prior distribution:

\begin{theorem}\label{thm:marginal_consistent}
Let \(\gamma:[0,1]\to[0,1]\) satisfy \(\gamma_0=1\) and \(\gamma_1=0\), and assume \(\gamma\) is monotone decreasing.
For any \(0<t\leq1\) and \(0\leq s <t\), define
\begin{align}
\gamma_{s\mid t} := \frac{1-\gamma_s}{1-\gamma_t}\in[0,1].
\label{eq:gamma_s_t_def}
\end{align}
Then, for every \(\lambda\in[0,1]\), the weights
\begin{align}
w_{\mathrm{stay}} &:= (1-\lambda)\,\gamma_{s\mid t}, \label{eq:w_stay}\\
w_{\mathrm{prior}}&:= \lambda\,(1-\gamma_s), \label{eq:w_prior}\\
w_{\mathrm{flip}} &:= (1-\lambda)\,(1-\gamma_{s\mid t})+\lambda\,\gamma_s \label{eq:w_flip}
\end{align}
satisfy the marginal constraints in \eqref{eq:weight_constraints}.
Therefore, the reverse transition can be written equivalently as
\begin{equation}
\begin{aligned}
p_\lambda(\x_s\mid \x_t,\x)
&=\Cat\!\Big(\x_s; (1-\lambda)\gamma_{s\mid t}\,\x_t
+\lambda(1-\gamma_s) \q_1 \\
&+\Big(1-\big((1-\lambda)\gamma_{s\mid t}+\lambda(1-\gamma_s)\big)\Big)\e_x\Big)
\label{eq:posterior_lambda_consistent}
\end{aligned}
\end{equation}

and its marginal satisfies
\begin{align}
\sum_{\x_t} p_\lambda(\x_s\mid \x_t,\x)\,p(\x_t\mid \x)
=
p(\x_s\mid \x),
\quad\text{for all }s<t,
\end{align}
where \(p(\x_t\mid \x)=\Cat(\x_t; (1-\gamma_t)\q_1+\gamma_t\e_\x)\).
\end{theorem}
We provide a proof for \cref{thm:marginal_consistent}
in \cref{sec:proof}.

\paragraph{Special cases.} The parameter $\lambda$ acts as a forget mechanism controlling how much probability mass is taken from the current state $\x_t$, and how much is resampled independently from the marginal. \cref{fig:transition_heatmaps} visualizes the transition distributions $p(\x_s \mid \x_t, \x)$ for increasing values of $\lambda$.
At $\lambda=0$, our model reduces to two choices, i.e., either remain in the current state, or flip to the data token. More specifically,
\begin{equation}
\begin{aligned}
p_0(\x_s\mid \x_t,\x) =\Cat\!\Big(\x_s; \gamma_{s\mid t}\,\x_t+\Big(1-\gamma_{s\mid t}\Big)\e_x\Big).
\end{aligned}
\end{equation}
This dynamics defines an absorbing diffusion process directed toward the data token $\e_\x$. In particular, when $\x_t=\e_\x$, the transition distribution collapses to $p_0(\x_{s}\mid \x_{t} = \e_\x,\x) =\Cat\!\big(\x_s; \e_\x\big)$ indicating that $\e_\x$ is absorbing and every other state $\x_t \neq \e_\x$ remains non-absorbing, with a nonzero probability of transitioning to $\e_\x$. 
This behavior is reflected in \cref{fig:transition_heatmaps}, where probability mass is concentrated along the diagonal and the column corresponding to the data point $\x$. For $\lambda=1$, the posterior reduces to:
\begin{equation}
\begin{aligned}
p_1(\x_s\mid \x_t,\x)
&=\Cat\!\Big(\x_s; (1-\gamma_s) \q_1 +\gamma_s\e_x\Big)\\
&=p(\x_s \mid \x).
\end{aligned}
\end{equation}
In this case, the latent variables $\x_t$ and $\x_s$ are conditionally independent given $\x$, i.e., the state $\x_t$ is entirely discarded, and the transitions reset to the marginal distribution. This is illustrated in \cref{fig:transition_heatmaps} where all rows collapse to the same distribution, reflecting complete forgetting of the current state $\x_t$.

\paragraph{Remark on the forward process.}
Our approach does not require instantiating a Markovian noising process $p(\x_t \mid \x_s)$, for $s < t$, following standard diffusion methods. Instead, we directly sample the noisy latents from the marginal distribution in \eqref{eq:marginal_main}.
For completeness, however, we derive the underlying forward transition step of \ours\ via Bayes' rule, $p_{\lambda}(\x_t\mid \x_s,\x)\propto p_{\lambda}(\x_s\mid \x_t,\x)p(\x_t\mid \x).$ This induces a non-Markovian forward kernel of the form:
\begin{equation}
    p_{\lambda}(\x_t \mid \x_s, \x) = \Cat\Big(\x_t; \alpha_{\x_s} \m_t + (1 - \alpha_{\x_s})\x_s\Big)\;, 
\end{equation}
where $\m_t = p(\x_t \mid \x)$ and $\alpha_{\x_s} \in [0,1]$ is a state-dependent resampling probability. This transition admits a simple interpretation: with probability $\alpha_{\x_s}$, the latent $\x_t$ is resampled from the marginal distribution, otherwise the state retains the previous state $\x_s$ with probability $1 - \alpha_{\x_s}$.
%The forward step can be interpreted as a mixture of retaining the previous state $\x_s$ and resampling from the marginal $p(\x_t \mid \x)$.
We provide a full derivation in \cref{app:forward_transition}.  Crucially, our marginal is independent of the resampling scheduler $\lambda$, allowing decoupled training and sampling.
%which generally depends on $\x$ and is therefore non-Markovian in $\x_t$ alone. 

\subsection{Generative Distribution and optimization}\label{sec:generative_model}
As the true sample $\x$ is not known during inference, we train a neural network as a surrogate. We define the generative distribution as follows:
\begin{align}
p_{\lambda}^\theta(\x_{s}\mid \x_{t})
&=\Cat\!\Big(\x_s; (1-\lambda)\big[\gamma_{s\mid t}\,\x_t+(1-\gamma_{s\mid t})\x_\theta\big] \notag \\
&\qquad
+\lambda\big[(1-\gamma_s) \q_1+\gamma_s\x_\theta\big]\Big), \label{eq:parametrized_reverse}
\end{align} 
where $\x_\theta \coloneqq f_\theta(\x_t, t)$ is a neural network prediction. As with standard diffusion models, we aim to optimize the log-likelihood $\log p_\theta(\x)$. However, due to the intractability of the marginalization, we derive the evidence lower bound (ELBO):
\begin{align}
    &\log p_{\lambda}^\theta(\x)
    \geq \notag \\
    &\quad \mathbb{E}_{p_\lambda(\x_{t(1)}, \dots,\x_{t(T)}\mid \x)}\left[
    \log \frac{p_\lambda^\theta(\x,\x_{t(1)}, \dots,\x_{t(T)})}{p_\lambda(\x_{t(1)}, \dots,\x_{t(T)}\mid \x)}
    \right] \notag \\
    &\quad=
    \mathbb{E}_{p_\lambda(\x_{t(1)}\mid \x)}\!\left[\log p_\lambda^\theta(\x\mid \x_{t(1)})\right] \notag \\
    \;&\qquad -\; \sum_{i=2}^{T}
\mathbb{E}_{p_\lambda(\x_{t(i)}\mid \x)}
\!\left[
\mathcal{L}_{\text{diff}}^{(i)}
\right]\; \notag\\
\;&\qquad -\;
    \mathbb{E}_{\q_0(\x)}\bigg[\mathrm{KL}\!\Big(
p(\x_{t(T)}\mid \x)\,\big\|\,p(\x_{t(T)}) \Big)\bigg]\;,
\end{align}
where
\[\mathcal{L}_{\text{diff}}^{(i)}=
 \mathrm{KL}\!\left( p_\lambda(\x_{t(i-1)}\mid \x_{t(i)},\x)\,\big\lVert\, p_\lambda^\theta(\x_{t(i-1)}\mid\x_{t(i)}) \right).
\]

The ELBO is maximized when the KL divergences between the two distributions are minimized, i.e., when the neural network $f_\theta$ predicts $\e_\x$. As the ELBOs share the same optimum across all $\lambda$ when the model perfectly predicts $\x$, we only need to train one model, allowing us to change $\lambda$ only during inference.
We provide the training and sampling algorithms of \ours\ in \cref{alg:cat-diffusion-training} and \cref{alg:cat-diffusion-sampling}, respectively.

\paragraph{Expected total number of state transitions.}
In the following, we analyze the expected number of state transitions induced by \ours, i.e., $\x_{t(k-1)} \neq \x_{t(k)}$, for $k \in \{1, \cdots, T\}$, along a generation trajectory of $T$ steps. Let
\[
C_k = \mathbf{1}\Big[\x_{t(k-1)} \neq \x_{t(k)}\Big]
\]
denote the indicator of a state change at step $k$. The expected total number of transitions is then given by:
\[
\mathbb{E}_{\lambda}\big[N_T\big] = 
\mathbb{E}_{\lambda}\Bigg[ \sum_{k=1}^{T} C_k\Bigg] = 
%\sum_{k=1}^{T} \mathbb{E}_{\lambda} \Big[C_k\Big] =
\sum_{k=1}^{T} \mathbb{E}_{\lambda} \Big[ p_{\lambda}(C_k = 1 \mid \x_{t(k)})\Big]. 
\]
In \cref{sec:expectation_of_transitions}, we show that this expectation admits the form $\mathbb{E}_{\lambda}\big[N_T\big] = A\lambda + B$, where $A,B\geq0$. This implies that the expected number of state transitions grows linearly with the resampling coefficient $\lambda$.
In practice, however, the model $f_{\theta}(\x_t, t)$ does not perfectly preserve the marginal, therefore the distribution of $\x_{t(k)}$ can drift with $\lambda$. We investigate this effect in \cref{sec:lambda_effect}.

\paragraph{Extension to graphs and sequences.}
So far, we have considered a single-token data point $\x$. \ours\ can be readily applied to arbitrary discrete domains and accommodates arbitrary prior distributions.
In this paper, we apply \ours\ to both graph and sequence generation tasks. We represent a graph $G = (\mathbf{X}, \mathbf{E})$ with $N$ nodes, where $\mathbf{X} \in \mathbb{R}^{N \times K_1}$ is a matrix of one-hot encoded representations of each node $\x_i \in \{0,1\}^{K_1}$, and $\mathbf{E} \in \mathbb{R}^{N \times N \times K_2}$ groups the one-hot encoding of each edge $\e_{ij} \in \{0,1\}^{K_2}$. Here, $K_1$ and $K_2$ denote the number of possible node and edge attributes, respectively. We adopt marginal prior distributions over nodes and edges, as they have proved effective in prior graph diffusion approaches \citep{vignac2022digress,qin2024defog}. For sequence generation, a sequence is represented as $\mathbf{S} \in \mathbb{R}^{L\times K}$, where $L$ is the sequence length and $K$ is the vocabulary size. In this work, we mainly focus on uniform priors, as we find that they consistently yield lower performance than absorbing-state diffusion models in language modeling \citep{austin2021structured,lou2023discrete}.

%% file: includes/transition_heatmaps.tex
\begin{figure}
    \centering
    \includegraphics[width=\columnwidth]{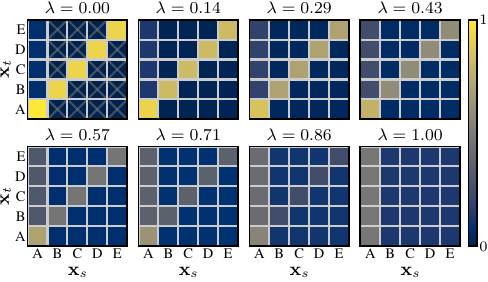}
    \caption{Visualization of $p(\x_s \mid \x_t, \x = \text{A})$ for different values of $\lambda$. Low $\lambda$ preserves the current state with high probability. As $\lambda$ increases, the mass progressively shifts toward resampling from the prior (uniform in this example).}
    \label{fig:transition_heatmaps}
\end{figure}

%% file: algorithms/algorithms.tex
\begin{figure*}[t]
\centering
\begin{minipage}[t]{0.49\textwidth}
\vspace{0pt}
\begin{algorithm}[H]
\caption{Training \ours}
\label{alg:cat-diffusion-training}
\begin{algorithmic}
\STATE \textbf{Inputs:} Data distribution $\q_0(\x)$; denoising model $f_\theta$; noise schedule $\gamma(\cdot)$; steps $T$; learning rate $\eta$

\STATE \textbf{Output:} trained parameters $\theta$

\STATE \textbf{While} not converged \textbf{do}

\STATE $\mathrm{d}t = 1/T$ 
\STATE $\x \sim \q_0(\x)$

\STATE $t \sim \mathcal{U}(\mathrm{d}t,1)$,\quad $s = t - \mathrm{d}t$

\STATE $\x_t \sim p(\x_t \mid \x)$

%\STATE $\gamma_{s\mid t} = (1 - \gamma_s) / (1 - \gamma_t)$

\STATE $p(\x_s \mid \x_t, \x) = \gamma_{s\mid t}\,\x_t + (1 - \gamma_{s\mid t})\,\x$

\STATE $p(\x_s \mid \x_t) = \gamma_{s\mid t}\,\x_t + (1 - \gamma_{s\mid t})\, f_\theta(\x_t, t)$

\STATE $\theta = \theta - \eta \nabla_\theta$ $\mathcal{L}_{\text{diff}}$

\STATE {\bfseries end while}

%\STATE \textbf{return} $\theta$
\end{algorithmic}
\end{algorithm}
\end{minipage}
%\hspace{1cm}
\hfill
\begin{minipage}[t]{0.49\textwidth}
\vspace{0pt}
\begin{algorithm}[H]
\caption{Sampling with \ours}
\label{alg:cat-diffusion-sampling}
\begin{algorithmic}
\STATE \textbf{Inputs:} Denoiser $f_\theta$; nosing scheduler $\gamma(\cdot)$; forgetting scheduler $\lambda(\cdot)$; steps $T$; prior $\q_1$

\STATE \textbf{Output:} $\x_0$

\STATE $\x_T \sim \q_1(\cdot)$

\STATE \textbf{for} $i=T, T-1, \ldots, 1$ \textbf{do}

\STATE \quad $t = i/T$,\quad $s = (i-1)/T$

\STATE \quad $\gamma_{s\mid t} = (1-\gamma_s)/(1-\gamma_t), \quad \lambda_t = \lambda(t)$

\STATE \quad $w_\text{stay} = (1 - \lambda) \gamma_{s\mid t}$, \quad $w_\text{prior} = \lambda (1 - \gamma_s)$

\STATE \quad $w_\text{flip} = 1 - (w_\text{stay} + w_\text{prior})$

\STATE \quad $\x_s \sim \Cat(\x_s;\,
w_\text{stay}\x_t + w_\text{prior}\q_1$
\STATE \qquad $\;+\;\, w_\text{flip} f_\theta(\x_t, t))$

\STATE\textbf{end for}
%\STATE \textbf{return} $\x_0$
\end{algorithmic}
\end{algorithm}
\end{minipage}
\end{figure*}

%% file: sections/4_experiments.tex
\section{Experiments}\label{sec:experiments}
In this section, we discuss our empirical evaluation across various benchmarks. Our primary goal is to compare our generative model to previous discrete diffusion approaches and investigate the effect of random flips to the prior distribution. We benchmark our model on two tasks: graph generation in \cref{sec:graph_experiments} and text generation in \cref{sec:language_experiments}. Additional details on the datasets, model architectures, training, and evaluation metrics are provided in \cref{app:exp_details}.

\subsection{Graph Generation}\label{sec:graph_experiments}
\paragraph{Experimental setup.}
For graph generation, we parametrize our neural network with the graph transformer of \citet{vignac2022digress}, using Relative Random Walk Probabilities (RRWP) \citep{ma2023graph,qin2024defog} to encode structural properties. The network predicts and generates node and edge labels (see ~\cref{sec:generative_model}). We provide additional details and hyperparameters in~\cref{app:training_settings}.
\paragraph{Small molecules.}
We evaluate our model on the QM9 dataset~\citep{wu2018moleculenet} (without hydrogen). The results and baselines are reported in \cref{tab:qm9_molecule_gen}.
\input{tables/qm9_results}
\input{tables/large_molecules}

Overall, \ours{} faithfully learns the data distribution, achieving a validity of 99.3\%, matching the training distribution, and achieves state-of-the-art FCD results with $\lambda=0$ and a time distortion scheduler with $\rho=4$, see \cref{sec:time_distortion} for more details.
\paragraph{Large molecules.}
In addition to QM9, we evaluate \ours{} on the two established benchmarks GuacaMol~\citep{brown2019guacamol} and MOSES~\citep{polykovskiy2020molecular}. We follow the standard data splits and provide our experimental results in \cref{tab:large_molecule_gen}.
We observe that our model excels in FCD, SNN, and Filters while maintaining a high validity. Notably, our model with reduced steps (10\% of steps) outperforms DeFoG with reduced steps across most metrics. Also, the full-step model achieves a higher validity than DeFoG. For these results, we used $\rho=4$ for both datasets, and $\lambda=0.4$ for MOSES and $\lambda=0.2$ for GuacaMol.
%\subsection{Protein Sequence Generation}

\subsection{Text Generation}\label{sec:language_experiments}

\paragraph{Experimental setting.} 
We train \ours\ on two datasets: the One Billion Words Dataset (LM1B) \citep{chelba2013one} and OpenWebText (OWT) \citep{Gokaslan2019OpenWeb}. We use the same data splits as \citet{wang2025remasking} and follow their tokenization and model configurations. 
For evaluation, we assess language modeling performance using likelihood-based metrics and report perplexity on the test split. To evaluate the quality of generated sequences, we compute generative perplexity (Gen PPL) using GPT-2 Large, and measure sentence entropy to assess output diversity. 

\paragraph{Likelihood evaluation.}
\input{includes/ppl_vs_T}
We compute the perplexity on LM1B and report results in \cref{tab:lm1b_ppl}. 
Our primary baselines are uniform discrete diffusion models, including D3PM (uniform) \citep{austin2021structured} and SEDD (uniform) \citep{lou2023discrete}. We further include an autoregressive baseline, as well as absorbing-state diffusion models such as D3PM (absorb) and MDLM \citep{sahoo2024simple}. 
Compared to prior diffusion baselines, \ours\ consistently improves performance within the uniform diffusion family, with SEDD (Uniform) \citep{lou2023discrete} achieving the lowest perplexity. As shown in \cref{fig:ppl_vs_T_lm1b}, we analyze perplexity scaling with respect to the number of diffusion steps. Using a fixed resampling parameter $\lambda = 0.2$, \ours\ attains strong performance with as few as 16 steps and outperforms both D3PM Uniform and D3PM Absorb by a substantial margin.
Additional results with different resampling parameters are presented in \cref{fig:ppl_vs_steps}.
\paragraph{Sample quality.} We assess the generation capability of \ours\ when training on OpenWebText. In \cref{tab:gen_ppl_entropy}, we report the generative perplexity and the entropy for generation steps in $\{1024,512,256\}$. We use $\lambda=0.2$ for sampling. \ours\ achieves comparable generative perplexity to SEDD (uniform) \citep{lou2023discrete}, while providing improved diversity. 
\input{tables/lm1b_results}
\input{tables/OWT_results}
%\todo{Elaborate}
\subsection{Time distortion}\label{sec:time_distortion}
\input{includes/ablation_4panel}
Following previous works, we distort time steps during sampling~\citep{karras2022elucidating,qin2024defog}. Instead of uniformly spacing the time steps, we allocate more steps to later stages of generation, i.e., closer to the final sample. More specifically, we use a similar scheduler as~\citep{karras2022elucidating}:
\begin{align*}
    t(i)'=t(i)^\rho,
\end{align*}
with $\rho>1$. We analyze the effect of $\rho$ in \cref{fig:ablation_moses} and report the FCD and validity for MOSES across different settings. Increasing $\rho$ above $1$ yields a substantial improvement while stagnating from $\rho=2$. We provide additional ablations in \cref{app:additional_results}.
\subsection{Effect of $\lambda$}\label{sec:lambda_effect}
Intuitively, increasing $\lambda$ introduces more transitions, i.e., $\x_t\neq\x_s$, during generation as the weights $w_{\mathrm{flip}}$ and $w_{\mathrm{prior}}$ increase. In the extreme case $\lambda=0$, the transition in \eqref{eq:general_posterior} reduces to a purely local update and will have at most one transition to the goal state. In practice, however, our trained neural network is not guaranteed to consistently predict the same class, which can lead to additional transitions.

To quantify this effect, we measure the average number of state transitions over $50$ sampling steps on QM9, for the nodes and edges. We show the results in \cref{fig:flips}.
\input{includes/flips_figure}
We observe that the number of transitions increases linearly with $\lambda$, starting close to 0. Note that the average number of transitions for $\lambda=0$ is well below 1, as we have a highly imbalanced class distribution and our sampling algorithm starts from the marginal distribution. We show its empirical effect on performance in \cref{fig:ablation_moses}.

%% file: tables/qm9_results.tex
\begin{table}[!h]
\centering
\caption{QM9 without hydrogen generation performance. Best results in \textbf{bold}, second best \underline{underlined}. Baseline results are taken from the corresponding papers.}\label{tab:qm9_molecule_gen}
\resizebox{\columnwidth}{!}{%
\begin{tabular}{l*{4}{c}}
\toprule
%\cmidrule(lr){2-6}\cmidrule(lr){7-13}
Model
& Val.\,$\uparrow$ & Relaxed Val. \,$\uparrow$ & Unique \,$\uparrow$ & FCD \,$\downarrow$\\
\midrule
Training set
& 99.3 & 99.5 & 99.2 & 0.03 \\
\midrule
DiGress{\tiny ~\citep{vignac2022digress}} & 99.0 & - & 96.2 & - \\
DisCo{\tiny~\citep{xu2024discrete}} & \underline{99.3} & - & - & -\\
Cometh{\tiny~\citep{siraudin2024cometh}} & \textbf{99.6} & - & \textbf{96.8} & 0.25\\
DeFoG{\tiny~\citep{qin2024defog} (10\%)} & 98.9 & 99.2 & 96.2 & 0.26\\
DeFoG{\tiny~\citep{qin2024defog}} & \underline{99.3} & \underline{99.4} & 96.3 & \underline{0.12} \\
\midrule
\textbf{\ours{} (Ours)} (10\%) & \underline{99.3} & \textbf{99.5} & \underline{96.5} & \textbf{0.10} \\
\textbf{\ours{} (Ours)} & \underline{99.3} & \textbf{99.5} & 96.3 & 0.13 \\
\bottomrule
\end{tabular}
}
\end{table}

%% file: tables/large_molecules.tex
\begin{table*}[t]
\centering
\caption{Large molecule generation performance. Best scores in \textbf{bold}, second best \underline{underlined}. Baseline results are taken from the corresponding papers.}
%\vspace{-0.5em}
\label{tab:large_molecule_gen}
\resizebox{\textwidth}{!}{%
\begin{tabular}{l*{11}{c}}
\toprule
& \multicolumn{4}{c}{GuacaMol} & \multicolumn{7}{c}{MOSES} \\
\cmidrule(lr){2-5}\cmidrule(lr){6-12}
Model
& Val.\,$\uparrow$ & V.U.\,$\uparrow$ & V.U.N.\,$\uparrow$ & FCD\,$\uparrow$
& Val.\,$\uparrow$ & Unique.\,$\uparrow$ & Novelty\,$\uparrow$ & Filters\,$\uparrow$ & FCD\,$\downarrow$ & SNN\,$\uparrow$ & Scaf\,$\uparrow$ \\
\midrule
Training set
& 100.0 & 100.0 & 0.0 & 92.8
& 100.0 & 100.0 & 0.0 & 100.0 & 0.01 & 0.64 & 99.1 \\
\midrule
DiGress{\tiny~\citep{vignac2022digress}} & 85.2 & 85.2 & 85.1 & 68.0
& 85.7 & \textbf{100.0} & 95.0 & 97.1 & \underline{1.19} & 0.52 & 14.8 \\
DisCo{\tiny~\citep{xu2024discrete}}
& 86.6 & 86.6 & 86.5 & 59.7
& 88.3 & \textbf{100.0} & \textbf{97.7} & 95.6 & 1.44 & 0.50 & 15.1 \\
Cometh{\tiny~\citep{siraudin2024cometh}}
& 98.9 & 98.9 & 97.6 & 72.7
& 90.5 & \underline{99.9} & 92.6 & \underline{99.1} & 1.27 & 0.54 & 16.0\\
DeFoG{\tiny~\citep{qin2024defog}}
& \underline{99.0} & \underline{99.0} & \textbf{97.9} & 73.8
& \textbf{92.8} & \underline{99.9} & 92.1 & 98.9 & 1.95 & \underline{0.55} & 14.4 \\
\textbf{\ours{} (Ours)} & \textbf{99.6} & \textbf{99.5} & \underline{97.8} &\textbf{78.3} & \underline{90.7} & \underline{99.9} & 90.2 & \textbf{99.4} & 1.58 & \textbf{0.56} & 13.2	 \\
\midrule
DeFoG{\tiny~\citep{qin2024defog} (10\% steps)}
& 91.7 & 91.7 & 91.2 & 57.9
& 83.9 & \underline{99.9} & \underline{96.9} & 96.5 & 1.87 & 0.50 & \textbf{23.5} \\
\textbf{\ours{} (Ours)} {\tiny(10\% steps)} & 98.1 & 98.0 & 97.5 & \underline{77.6} & 87.1 & \underline{99.9} & 94.3 & 98.9	 & \textbf{1.14} & 0.53 & \underline{18.5} \\
\bottomrule
\end{tabular}
}
\vspace{1em}
\end{table*}

%% file: includes/ppl_vs_T.tex
{\setlength{\columnsep}{6pt}
\begin{wrapfigure}[12]{r}{1.75in}
  \centering
  \vspace{-1\baselineskip} 
  \includegraphics[width=1.75in]{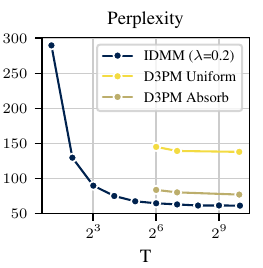}
  \vspace{-1.75em}
  \caption{Perplexity scaling.}
  \label{fig:ppl_vs_T_lm1b}
  %\vspace{-0.8\baselineskip} % tweak to reduce whitespace after
\end{wrapfigure}

%% file: tables/lm1b_results.tex
\begin{table}[htpb]
\centering
\caption{Test perplexities (PPL) on LM1B.
The best value is \underline{underlined}, and the best uniform diffusion model is \textbf{bolded}.
 $^\dagger$, $^{*}$, and $^{\$}$ indicate results taken from \citet{sahoo2025diffusion}, \citet{lou2023discrete}, and \citet{austin2021structured}, respectively.}
\label{tab:lm1b_ppl}
\resizebox{\columnwidth}{!}{%
\begin{tabular}{l
  S[table-format=3.1]
  %S[table-format=3.1, table-space-text-post=-]
  %S[table-format=3.1, table-space-text-post=-]
  }
\toprule
Model 
 & \multicolumn{1}{c}{PPL\,$\downarrow$} \\
%\cmidrule(lr){2-2}
%& Model
%& {\(T=1000\)}
%& {\(T=128\)}
%& {\(T=64\)} 
%\\
\midrule
\textit{Autoregressive} $^\dagger$ \\
\quad Transformer \citep{sahoo2024simple}
& \underline{22.3} 
%& \multicolumn{1}{c}{-} 
%& \multicolumn{1}{c}{-} 
\\
\midrule
\textit{Diffusion (Absorbing State / Gaussian)} $^\dagger$ \\
\quad BERT-Mouth \citep{wang2019bert} 
& 142.9 
%& \multicolumn{1}{c}{-} 
%& \multicolumn{1}{c}{-}
\\
\quad D3PM Absorb$^{\$}$ \citep{austin2021structured} 
& 76.9 
%& 80.1 
%&  83.6 
\\
\quad Diffusion-LM \citep{li2022diffusion} & 118.6
%& \multicolumn{1}{c}{-} 
%& \multicolumn{1}{c}{-} 
\\
\quad DiffusionBert \citep{he2023diffusionbert} 
& 63.8 
%& \multicolumn{1}{c}{-} 
%& \multicolumn{1}{c}{-} 
\\
\quad SEDD Absorb \citep{lou2023discrete}
& 32.7 
%& \multicolumn{1}{c}{-} 
%& \multicolumn{1}{c}{-} 
\\
\quad MDLM \citep{sahoo2024simple}
& \textbf{27.0} 
%& \multicolumn{1}{c}{-} 
%& \multicolumn{1}{c}{-}
\\
\quad Duo \citep{sahoo2025diffusion} & 29.9 
%& \multicolumn{1}{c}{-} 
%& \multicolumn{1}{c}{-}
\\
\quad \textbf{\ours\ (Ours)} & 42.2 \\

\midrule
\textit{Diffusion (Uniform)} \\
\quad D3PM Uniform$^{\$}$ \citep{austin2021structured} 
& 137.9 
%& 139.2 
%& 145.0 
\\
\quad SEDD Uniform$^{*}$ \citep{lou2023discrete} & \textbf{40.3} 
%& \multicolumn{1}{c}{-} 
%& \multicolumn{1}{c}{-} 
\\
\quad \textbf{\ours\ (Ours)} 
& 51.1 
%& \textbf{\underline{62.8}} 
%& \textbf{\underline{64.5}} 
\\
\bottomrule
\end{tabular}
}
\vspace{-1.2em}
\end{table}

%% file: tables/OWT_results.tex
\begin{table*}[htpb]
\caption{Generative perplexity using GPT-2 large and entropy for different generation steps $T$. Best values are \underline{underlined}, and the best uniform diffusion results are \textbf{bolded}.}
\small
\centering
\begin{tabular}{l
S[table-format=3.1, table-space-text-post=-]
S[table-format=3.1, table-space-text-post=-]
S[table-format=3.1, table-space-text-post=-]
S[table-format=1.2, table-space-text-post=-]
S[table-format=1.2, table-space-text-post=-]
S[table-format=1.2, table-space-text-post=-]
}
\toprule 
& \multicolumn{3}{c}{Gen PPL.\,$\downarrow$} 
& \multicolumn{3}{c}{Entropy\,$\uparrow$} \\
\cmidrule{2-7}
Model & {\(T=1024\)} & {\(T=512\)} & {\(T=256\)} 
& {\(T=1024\)} & {\(T=512\)} & {\(T=256\)}  \\

\midrule
Autoregressive  & 12.1 &  \text{-} &  \text{-} & 5.22 &  \text{-} &  \text{-} \\
\midrule 
SEDD Absorb \citep{lou2023discrete} & 105.0 &  104.5 & 109.8 & 5.62 & 5.62 & 5.63\\
MDLM \citep{sahoo2024simple} & 104.9 & 104.4 & 112.7 & \textbf{5.63} & 5.63 & \textbf{5.66}\\
Duo \citep{sahoo2025diffusion} & \textbf{77.7} & \textbf{78.1} &  \textbf{78.6} & 5.55 & 5.55 & 5.55\\
\midrule
SEDD Uniform \citep{lou2023discrete} & \underline{99.9} & \underline{100.4} & 103.4 & 5.56 & 5.56 &  \underline{5.56} \\
\textbf{\ours{} (ours)} & 100.9 & 103.2 &\underline{98.5}  & \underline{5.60} & \underline{\textbf{5.64}}&  5.51\\
\bottomrule
\end{tabular}
\label{tab:gen_ppl_entropy}
\vspace{1.3em}
\end{table*}

%% file: includes/ablation_4panel.tex
\begin{figure*}[t]
    \centering
    \vspace{0.5em}
    \includegraphics[width=\linewidth]{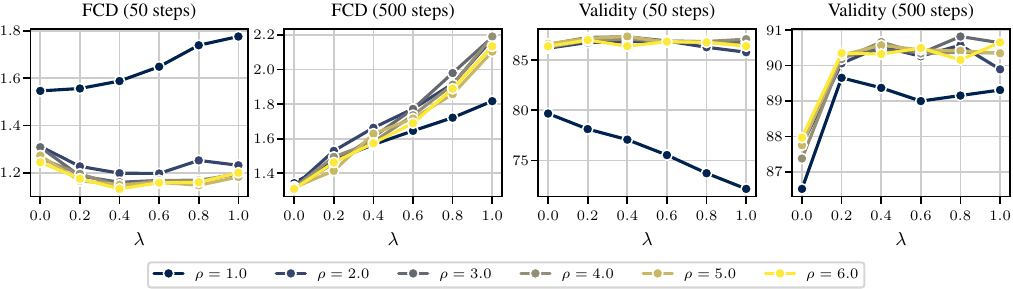}
    \caption{Different metrics on the MOSES dataset across different values of $\rho$ and $\lambda$. For 50 steps, increasing $\rho$ proves to be helpful but yields vanishing improvements after $\rho=2$. Further, the FCD and validity achieve their optimum around 0.4 for 50 steps. Finally, for 500 steps, increasing $\lambda$ worsens the FCD while increasing the validity.}
    \label{fig:ablation_moses}
    \vspace{0.5em}
\end{figure*}

%% file: includes/flips_figure.tex
\begin{figure}[h]
    \centering
    %\vspace{0.0em}
    \includegraphics[width=\columnwidth]{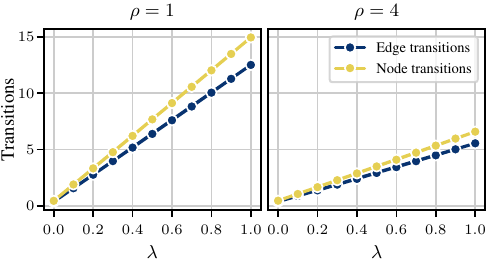}
%\vspace{-1.5em}
    \caption{Number of transitions across different $\lambda$. The number of transitions rises linearly when increasing $\lambda$ for nodes and edges on the QM9 dataset.}
    \label{fig:flips}
    \vspace{-1.7em}
\end{figure}

%% file: sections/5_related_work.tex
\section{Related Work}
\paragraph{Discrete Diffusion models.}
Diffusion models were first introduced for continuous data ~\citep{sohl2015deep}, and have demonstrated strong generative performance across a range of domains, including images~\citep{nichol2021improved,ho2020denoising}, audio~\citep{kong2020diffwave}, and video~\citep{ho2022video}. 
The extension of diffusion models to discrete data was first explored by~\citet{sohl2015deep}, which adopted a binomial diffusion process for binary data. Building on this,~\citet{hoogeboom2021argmax} generalized the framework to categorical random variables using transition matrices with uniform corruption probabilities.~\citet{austin2021structured} propose more structured categorical corruption processes for certain data types, such as absorbing-state diffusion for sequences, and discretized Gaussian kernels for ordinal data. 

Recent works~\citep{sahoo2024simple} examine a special case of the corruption process in~\citet{austin2021structured}, reformulating it as an interpolation between clean data and a prior distribution. In particular, \citet{sahoo2024simple} focus on masked diffusion models, and derive a simplified expression of the ELBO, showing that it corresponds to a weighted average of cross-entropy losses. \citet{schiff2024simple} extend this to uniform noise models and provide a continuous-time NELBO leading to performance gains. \citet{sahoo2025diffusion} propose a hybrid model that combines a discrete diffusion process with a Gaussian diffusion process defined directly over one-hot token representations rather than continuous embeddings. Alternative perspectives on discrete diffusion leverage continuous-time Markov chains (CTMC)~\citep{campbell2022continuous,campbell2024generative} or concrete score matching~\citep{lou2023discrete}.

\paragraph{Graph generation.}
Graph generation approaches can be categorized into two classes: autoregressive approaches, which generate nodes and edges sequentially, and one-shot approaches, which rely on diffusion-based methods.
Diffusion models leverage the order-agnostic nature of denoising processes to maintain permutation equivariance, and many state-of-the-art approaches further incorporate marginal distributions over node and edge attributes as priors~\citep{vignac2022digress,qin2024defog}.
Recent advances in graph generation have demonstrated that both discrete diffusion models, implemented in either discrete-time settings~\citep {vignac2022digress} or continuous-time formulations~\citep{xu2024discrete}, as well as discrete flow-matching approaches~\citep{qin2024defog}, achieve strong performance in different domains, including molecular and synthetic graph generation. \citet{boget2025simple} propose a diffusion model for graphs by assuming that the intermediate noisy states only depend
on the clean data, thereby removing the dependence on partially denoised states. This is closely related to \ours{}'s extreme case where the resampling coefficient is equal to 1.

\paragraph{Text generation.}
For language modeling, autoregressive transformers~\citep{radford2019language,vaswani2017attention} have achieved impressive results. However, their sequential token sampling leads to slow inference and limited controllability. To address this, various works explore discrete diffusion models for language generation. In particular, absorbing-state diffusion models \citep{shi2024simplified,sahoo2024simple} have demonstrated strong empirical results; however, they suffer from a key limitation: once a token is unmasked, it cannot be revised, which can cause error accumulation and degrade sample quality. Several predictor-corrector samplers have been proposed \citep{campbell2022continuous,gat2024discrete,wang2025remasking} to mitigate this issue, but at the cost of increased computation. In contrast, uniform diffusion models \citep{austin2021structured,schiff2024simple} naturally exhibit a self-correcting property absent in autoregressive and masked diffusion approaches. However, they do not achieve state-of-the-art performance. We note that our \ours{} approach naturally enables resampling from arbitrary priors, thereby allowing refinement of tokens for both absorbing-state and uniform diffusion. 

\section{Discussion} 
\paragraph{Limitations.}
We evaluate our model on graph and text generation. However, our proposed framework applies directly to any other discrete domain, including protein generation~\citep{campbell2024generative} and DNA sequences~\citep{stark2024dirichlet}. Further, our experiments are limited to uniform and marginal priors. Our generative model, however, readily extends to masking-based priors, thereby enabling iterative token refinement.

\paragraph{Future work.} Beyond extending to different domains and prior distributions, future work could investigate the use of guidance and other conditioning techniques. Finally, it would be promising to explore alternative training techniques based on variance reduction or distillation ~\citep{sahoo2025diffusion}.
%\todo{Sirine read please}

%\begin{itemize}
%    \item Unlike masked and uniform diffusion models, where the reverse posterior is fully determined by a fixed Markovian forward transition kernel, we only constrain the marginal distribution $p(\x_t \mid \x)$.
%    \item 
%\end{itemize}

%\textcolor{red}{TODO:}
%\begin{itemize}
%    \item While our generative model can be applied with arbitrary prior distributions, we mainly focus on uniform and marginal priors in our experiments in \cref{sec:experiments}, and show that we can outperform different state-of-the-art uniform diffusion models.
%\end{itemize}

%% file: sections/6_conclusion.tex
\section{Conclusion}
In this work, we introduced a discrete diffusion model whose reverse process is expressed by three actions: \emph{(i)} staying in the current state, \emph{(ii)} transitioning to the goal state, and \emph{(iii)} resampling from the prior distribution. By defining a marginal distribution, we derive a family of posteriors that interpolates between a local and a state-independent update, enabling us to directly control the transition frequency, while only training a single model. Further, we show that our reverse process cannot be recovered from a standard Markov noising process, as it corresponds to a non-Markovian forward process. Finally, we demonstrate that \ours{} achieves competitive performance on graph and language generation tasks.

\section*{Broader Impact}
In this work, we propose a generative model for discrete data and apply it to molecule and text generation. As discrete generative models can be applied in various fields, we see potential positive and negative societal impacts. 
Generative modeling can accelerate scientific discovery and enhance natural language generation capabilities. However, there are potential societal consequences, specifically those related to the generation of harmful text or designing unsafe drugs.
%However, we do not see any immediate concerns.

%% file: sections/7_appendix.tex
\section{Proof of \texorpdfstring{\cref{thm:marginal_consistent}}{Theorem~\ref{thm:marginal_consistent}}}\label{sec:proof}

\begin{theorem}\label{thm:marginal_consistent_app}
Let \(\gamma:[0,1]\to[0,1]\) satisfy \(\gamma_0=1\) and \(\gamma_1=0\), and assume \(\gamma\) is monotone decreasing.
For any \(0<t\leq1\) and \(0\leq s <t\), define
\begin{align}
\gamma_{s\mid t} := \frac{1-\gamma_s}{1-\gamma_t}\in[0,1].
\end{align}
Then for every \(\lambda\in[0,1]\), the weights
\begin{align}
w_{\mathrm{stay}} &:= (1-\lambda)\,\gamma_{s\mid t}\label{eq:w_stay_app} , \\
w_{\mathrm{prior}}&:= \lambda\,(1-\gamma_s),\\
w_{\mathrm{flip}} &:= (1-\lambda)\,(1-\gamma_{s\mid t})+\lambda\,\gamma_s \label{eq:w_flip_app}
\end{align}
satisfy the marginal constraints \eqref{eq:weight_constraints}.
Therefore, the reverse transition can be written equivalently as
\begin{equation}
\begin{aligned}
p_\lambda(\x_s\mid \x_t,\x)
&=\Cat\!\Big(\x_s; (1-\lambda)\gamma_{s\mid t}\,\x_t
+\lambda(1-\gamma_s) \q_1 \\
&+\Big(1-\big((1-\lambda)\gamma_{s\mid t}+\lambda(1-\gamma_s)\big)\Big)\e_x\Big)
\end{aligned}
\end{equation}

and its marginal satisfies
\begin{align}
\sum_{\x_t} p_\lambda(\x_s\mid \x_t,\x)\,p(\x_t\mid \x)
=
p(\x_s\mid \x),
\qquad\text{for all }s<t,
\end{align}
where \(p(\x_t\mid \x)=\Cat(\x_t; (1-\gamma_t)\q_1+\gamma_t\e_\x)\).
\end{theorem}

\begin{proof}
We provide a proof by induction.

As a starting point, we have $p(\x_{t_1}\mid \x)=\Cat((1-\gamma_{t_1})\q_1+\gamma_{t_1}\e_\x)$ by definition.
For the inductive hypothesis, we assume that for some $t \in [0,1]$, the marginal satisfies $p(\x_t\mid \x)=\Cat((1-\gamma_t)\q_1+\gamma_t\e_\x)$.
The goal is to show that $p(\x_s\mid \x)=\Cat((1-\gamma_s)\q_1+\gamma_s\e_\x)$, for $s < t$.

Let $\m_t = (1-\gamma_t)\q_1+\gamma_t\e_\x$.
By marginalizing out $\x_t$ and using the linearity of the expectation, we have:
\begin{align}
p(\x_s\mid \x)
&=\sum_{\x_t} p_\lambda(\x_s\mid \x_t,\x)\,p(\x_t\mid \x)\notag\\
&=\sum_{\x_t} \Cat(\x_s; w_\text{stay}\x_t + w_\text{prior}\q_1 + w_\text{flip}\e_x)\,\Cat(\x_s; \m_t)\notag\\
&=\Cat\Big(\x_s; \mathbb{E}_{\x_t \sim p(\cdot \mid \x)} [w_\text{stay}\x_t + w_\text{prior}\q_1 + w_\text{flip}\e_x] \Big)\notag\\
&=\Cat\Big(\x_s; w_\text{stay}\m_t + w_\text{prior}\q_1 + w_\text{flip}\e_x \Big) \label{eq:weighted_mt}
\end{align}

Now we simplify the expressions using the coefficients in \eqref{eq:w_stay_app} and \eqref{eq:w_flip_app}:
\begin{align}
w_{\text{stay}} \m_t
&= (1-\lambda)\,\gamma_{s\mid t}\bigl((1-\gamma_t)\q_1+\gamma_t\e_x\bigr) \notag\\
&= (1-\lambda)\frac{1-\gamma_s}{1-\gamma_t}(1-\gamma_t)\q_1
   + (1-\lambda)\gamma_{s\mid t}\gamma_t\e_x \notag\\
&= (1-\lambda)(1-\gamma_s)\q_1
   + (1-\lambda)\gamma_{s\mid t}\gamma_t\e_x \label{eq:w_stay_mt}
\intertext{\vspace{0.5ex}}
w_{\text{flip}}\e_x
&= \bigl((1-\lambda)(1-\gamma_{s\mid t})+\lambda\gamma_s\bigr)\e_x \notag\\
&= (1-\lambda)(1-\gamma_{s\mid t})\e_x
   + \lambda\gamma_s\e_x \label{eq:w_flip_ex}
\end{align}

Substituting \eqref{eq:w_stay_mt} and \eqref{eq:w_flip_ex} into \eqref{eq:weighted_mt} yields:
\begin{align}
p(\x_s\mid \x)
&=\Cat\Big(\x_s; w_\text{stay}\m_t + w_\text{prior}\q_1 + w_\text{flip}\e_x \Big) \notag \\
&=\Cat\Big(\x_s;(1-\lambda)(1-\gamma_s)\q_1
   + (1-\lambda)\gamma_{s\mid t}\gamma_t\e_x + \lambda(\1-\gamma_s) \q_1 + (1-\lambda)(1-\gamma_{s\mid t})\e_x + \lambda\gamma_s\e_x\Big)  \notag \\
   &=\Cat\Big(\x_s; (1-\gamma_s) \q_1 + \big( (1-\lambda)(1-\gamma_{s \mid t}(1-\gamma_t)) + \lambda \gamma_s\big) \e_\x\Big) \notag \\
   &=\Cat\Big(\x_s;(1-\gamma_s) \q_1 + \big( (1-\lambda)(1-\frac{1-\gamma_s}{1-\gamma_t}(1-\gamma_t)) + \lambda \gamma_s\big) \e_\x\Big) \notag \\
   &=\Cat\Big(\x_s;(1-\gamma_s) \q_1 + \big( (1-\lambda)\gamma_s + \lambda \gamma_s\big) \e_\x \Big) \notag \\
   &=\Cat\Big(\x_s;(1-\gamma_s) \q_1 + \gamma_s \e_\x\Big) \notag
\end{align}
This completes the induction.
\end{proof}

\section{Derivation of the forward transition \texorpdfstring{$p_{\lambda}(\mathbf{x}_t \mid \mathbf{x}_s, \mathbf{x})$}{p\_lambda(x\_t | x\_s, x)}}%
\label{app:forward_transition}
In the following, we derive the forward transition step $p_{\lambda}(\x_t \mid \x_s, \x)$. 
Recall our reverse process defined for $\lambda \in [0,1]$:
\begin{align}
    p_\lambda(\x_s\mid \x_t,\x) 
    &=\Cat\!\Big(\x_s; (1-\lambda)\gamma_{s\mid t}\,\x_t +\lambda(1-\gamma_s) \q_1 \\ 
    &+\Big(1-\big((1-\lambda)\gamma_{s\mid t}+\lambda(1-\gamma_s)\big)\Big)\e_x\Big) \\
    &=\Cat\!\Big(\x_s; w_\text{stay}\x_t + w_\text{prior}\q_1 + w_\text{flip}\e_\x\Big) 
\end{align}
Let $\x_t$ and $\e_\x$ be single tokens. We write the categorical parameters in scalar form. 
The reverse kernel implies that, for any $i,j \in \{1, \dots K\}$, 
\[
p_\lambda(\x_s = i \mid \x_t = j, \e_\x) = w_\text{stay} \mathbf{1}[i=j] + w_\text{prior} \q_1(i) + w_\text{flip} \mathbf{1}[i = \x].
\]
Let $\m_t(j) = p(\x_t = j \mid \x)$, where $\m_t(\cdot) = (1 - \gamma_t) \q_1(\cdot) + \gamma_t \mathbf{1}[. = \x]$.

We first derive the unnormalized forward probabilities:
We fix $\x_s = i$. Then, for each $j$, we have:
\begin{align}
p_{\lambda}(\x_t = j \mid \x_s = i, \x) &\propto  p_{\lambda}(\x_s = i \mid \x_t = j, \x) p(\x_t = j \mid \x) \\
&\propto \big( w_\text{stay} \mathbf{1}[i=j] + w_\text{prior} \q_1(i) + w_\text{flip} \mathbf{1}[i = \x] \big) \m_t(j) \\
&\propto  \big( w_\text{stay} \mathbf{1}[i=j] + D_i \big) \m_t(j), \label{eq:unnormalized_forward_step}
\end{align}
where $D_i = w_\text{prior} \q_1(i) + w_\text{flip} \mathbf{1}[i = \x]$.

%If $j \neq i$, then $p_{\lambda}(\x_t = j \mid \x_s = i, \x) = D_i \m_t(j)$. \\
%If $j = i$, then $p_{\lambda}(\x_t = i \mid \x_s = i, \x) = (w_\text{stay} + D_i)\m_t(i)$.
%\begin{align}
%Z_i &= \sum_{j'} p_{\lambda}(\x_s = i \mid \x_t = j', \x) p(\x_t = j' \mid \x)\\
%&= \sum_{j'} (w_\text{stay} \mathbf{1}[i = j'] + D_i) \m_t(j')\\
%&= w_\text{stay}\m_t(i) + D_i \sum_{j'} \m_t(j') \\
%&= w_\text{stay}\m_t(i) + D_i\;,
%\end{align}
%since $ \sum_{j'}\m_t(j') = 1$. 
The normalizing constant is $p(\x_s = i \mid \x) = \m_s(i)$. We can also derive it as follows:
\begin{align}
p(\x_s =i \mid \x) &= \sum_{j'} p_{\lambda}(\x_s = i \mid \x_t = j', \x) p(\x_t = j' \mid \x)\\
&= \sum_{j'} (w_\text{stay} \mathbf{1}[i = j'] + D_i) \m_t(j')\\
&= w_\text{stay}\m_t(i) + D_i \sum_{j'} \m_t(j') \\
&= w_\text{stay}\m_t(i) + D_i\;,
\end{align}
since $ \sum_{j'}\m_t(j') = 1$. The normalized forward transition of \eqref{eq:unnormalized_forward_step} is:
\begin{equation}\label{eq:forward_bayes_expression}
    p_{\lambda}(\x_t = j \mid \x_s = i, \x) = 
    \frac{\big(w_\text{stay}\mathbf{1}[i=j] + D_i\big)\m_t(j)}{\m_s(i)}.
\end{equation}

If $i = j$, then:
\[p_{\lambda}(\x_t = i \mid \x_s = i, \x) = \frac{(w_\text{stay} + D_i)\m_t(i)}{\m_s(i)}.
\]
For $i \neq j$:
\[
p_{\lambda}(\x_t = j \mid \x_s = i, \x) = \frac{D_i \m_t(j)}{\m_s(i)}.
\]

To further simplify \eqref{eq:forward_bayes_expression}, we define the coefficient:
\begin{equation}\label{eq:alpha}
  \alpha_i = \frac{D_i}{\m_s(i)} = 
  \frac{D_i}{w_\text{stay}\m_t(i) + D_i} \in [0,1]\;.
\end{equation}
Then, we can rewrite the distribution over $\x_t$ as a mixture:
\begin{equation}
    p_{\lambda}(\x_t = j \mid \x_s = i, \x) = \alpha_i \m_t(j) + (1 - \alpha_i) \mathbf{1}[j = i] \;.
\end{equation}
The final forward transition can be simplified to:
\begin{equation}\label{eq:forward_transition_lambda}
    p_{\lambda}(\x_t \mid \x_s, \x) = \Cat(\x_t; \alpha_{\x_s} \m_t + (1 - \alpha_{\x_s})\e_{\x_s})\;,
\end{equation}
which means that in the forward transition, $\x_t = \x_s$ with probability $(1-\alpha_{\x_s})$, otherwise we sample from the marginal $\x_t \sim \m_t = p(\x_t \mid \x)$. 
The underlying forward process of \ours\ is therefore non-Markovian as it requires access to the clean state $\x$.

\paragraph{Instantiating the forward transitions for the special cases.}
\paragraph{For $\lambda = 0$:} The coefficients in the reverse process are simplified to:
\[
w_\text{stay}=\gamma_{s\mid t}, \quad w_\text{prior}=0,  \quad \text{and } w_\text{flip} = 1 - \gamma_{s\mid t}.
\]
Therefore, $D_i = (1 - \gamma_{s\mid t}) \mathbf{1}[i = x]$.
The forward transition is:
\begin{itemize}
    \item \textbf{Case 1:} if $\x_s \neq \x$ i.e., $(i \neq x)$, then $D_i = 0$ and the mixing coefficients $\alpha_i = 0$ according to \eqref{eq:alpha}. If we plug in these terms to \eqref{eq:forward_transition_lambda}, we get \\
    \[p_0(\x_t \mid \x_s, \x) = \Cat(\x_t; \e_\x)\;.
    \]
    \item \textbf{Case 2:} if $\x_s = \x$, then $D_i = (1 - \gamma_{s\mid t})$, and the coefficient $\alpha$ in \eqref{eq:forward_transition_lambda} simplifies to \[
    \alpha = \frac{1 - \gamma_{s\mid t}}{(1 - \gamma_{s\mid t} + \gamma_{s \mid t}\m_t)}\;.
    \]
\end{itemize}

\paragraph{For $\lambda = 1$:} The coefficients of the generative process are
\[w_\text{stay}=0, \quad w_\text{prior}=(1 - \gamma_s), \quad \text{and } w_\text{flip} = \gamma_s\]. 
Therefore, the Bayes expression in \eqref{eq:forward_bayes_expression} collapses to:
\[
p_1(\x_t = j \mid \x_s = i, \x) = \frac{D_i \m_t(j)}{D_i} = \m_t(j)\;.
\]
This implies that the forward transition is \textbf{conditionally independent of $\x_s$ given $\x$} and corresponds directly to the marginal
\[
p_1(\x_t \mid \x_s , \x) = p(\x_t \mid \x) = \Cat\Big((1 - \gamma_t)\q_1 + \gamma_t \e_\x \Big).
\]

%\subsection{Training Signal for different $\lambda$}
%We can write $p_\lambda$ as:
%\begin{align}
%p_\lambda(\x_s\mid \x_t,\x)
%&=
%\lambda
%\underbrace{\Cat\Big(\x_s;\gamma_{s\mid t}\x_t + (1-\gamma_{s\mid t})\e_\x\Big)}_{=:~p_{\mathrm{loc}}(\x_s\mid \x_t,\x)}
%+(1-\lambda)\,
%\underbrace{\Cat\Big(\x_s;(1-\gamma_s)\q_1 + \gamma_s\e_\x\Big)}_{=:~p_{\mathrm{marg}}(\x_s\mid \x)}.
%\end{align}
%Now, we look at the gradients of the KL divergence $\text{KL}(p_\lambda(\x_s\mid \x_t,\x)\mid p_\lambda^\theta(\x_s\mid \x_t))$:
%\begin{align}
%    &\nabla_\theta\text{KL}(p_\lambda(\x_s\mid \x_t,\x)\mid p_\lambda^\theta(\x_s\mid \x_t)) \\
%    &= - \nabla_\theta\mathbb{E}_{p_\lambda(\x_s\mid \x_t,\x)}[\log p_\lambda^\theta(\x_s\mid \x_t)]\\ 
%    &= - \lambda\nabla_\theta\mathbb{E}_{p_{\mathrm{loc}}(\x_s\mid \x_t,\x)}[\log p_\lambda^\theta(\x_s\mid \x_t)] - (1-\lambda)\nabla_\theta\mathbb{E}_{p_{\mathrm{marg}}(\x_s\mid \x)}[\log p_\lambda^\theta(\x_s\mid \x_t)] \\
%    &= - \lambda\nabla_\theta\mathbb{E}_{p_{\mathrm{loc}}(\x_s\mid \x_t,\x)}[\log p_\lambda^\theta(\x_s\mid \x_t)] - (1-\lambda)\nabla_\theta\mathbb{E}_{p_{\mathrm{marg}}(\x_s\mid \x)}[\log p_\lambda^\theta(\x_s\mid \x_t)] \\
%    &= - \lambda\gamma_{s\mid t}\nabla_\theta\log p_\lambda^\theta(\x_t\mid \x_t) - \lambda(1-\gamma_{s\mid t})\nabla_\theta\log p_\lambda^\theta(\e_x\mid \x_t) \\ &- (1-\lambda)(1-\gamma_s)\nabla_\theta\mathbb{E}_{\q_1}[\log p_\lambda^\theta(\x_s\mid \x_t)] - (1-\lambda)\gamma_s\nabla_\theta\log p_\lambda^\theta(\e_x\mid \x_t)
%\end{align}

\section{Number of expected transitions}\label{sec:expectation_of_transitions}
We can analytically derive the expected number of transitions, i.e., number of transitions where $\x_s\neq\x_t$.
First, we recall the update probability:
\begin{align}
    p_\lambda(\x_s=j\mid\x_t=i,\x)=(1-\lambda)\gamma_{s\mid t}\mathbf{1}[i = j]+\lambda(1-\gamma_s)\q_1(j)+w_\mathrm{flip}\mathbf{1}[j = \x].
\end{align}
A transition occurs when $j\neq i$. We define the indicator variable $C=\mathbf{1}[\x_t\neq\x_s]$. Now we want to derive the probability of a transition:
\begin{align}
    p_\lambda(C=1\mid\x_t=i,\x)=1-p_\lambda(\x_s=i\mid\x_t=i,\x).
\end{align}
Next, we condition the probability $p_\lambda(\x_s=j\mid\x_t=i,\x)$ on two events:
\paragraph{Case 1: $\x_t\neq \x$} In this case, the last term drops out out and the transition simplifies further:
\begin{align}
    p_\lambda(\x_s=\x_t\mid\x_t\neq \x)&=(1-\lambda)\gamma_{s\mid t}+\lambda(1-\gamma_s)\q_1(\x_t), \\
\end{align}
and therefore:
\begin{align}
    p_\lambda(C=1\mid\x_t\neq \x)&=1-(1-\lambda)\gamma_{s\mid t}-\lambda(1-\gamma_s)\q_1(\x_t), \\
    &=1-\gamma_{s\mid t}+\lambda\gamma_{s\mid t}-\lambda(1-\gamma_s)\q_1(\x_t) \\
    &= A + \lambda\left(\frac{1-\gamma_{s}}{1-\gamma_{t}}-(1-\gamma_s)\q_1(\x_t)\right) \\
    &= \underbrace{A}_{\geq 0} + \lambda\underbrace{\left((1-\gamma_s)\left(\frac{1}{1-\gamma_{t}}-\q_1(\x_t)\right)\right)}_{\geq 0} 
\end{align}
\paragraph{Case 2: $\x_t=\x$} Here, we plug in the value of $w_\mathrm{flip}$:
\begin{align}
    p_\lambda(\x_s=\x_t\mid\x_t=\x)&=(1-\lambda)\gamma_{s\mid t}+\lambda(1-\gamma_s)\q_1(\x)+(1-\lambda)\,(1-\gamma_{s\mid t})+\lambda\,\gamma_s
\end{align}
And therefore:
\begin{align}
p_\lambda(C=1\mid\x_t=\x)&= 1 -(1-\lambda)\gamma_{s\mid t}-\lambda(1-\gamma_s)\q_1(\x)-(1-\lambda)\,(1-\gamma_{s\mid t})-\lambda\,\gamma_s \\
    &= 1 -(1-\lambda)-\lambda(1-\gamma_s)\q_1(\x)-\lambda\,\gamma_s \\
    &= \lambda-\lambda(1-\gamma_s)\q_1(\x)-\lambda\,\gamma_s \\
    &= \lambda\left(1-(1-\gamma_s)\q_1(\x)-\gamma_s\right) \\
    &= \lambda\underbrace{\left((1-\gamma_s)(1-\q_1(\x)\right)}_{\geq 0} \\
\end{align}
Finally, we write down the number of total expected transitions:
\begin{align}
    \mathbb{E}\left[\sum_{t=1}^T C_t\right]&=\sum_{t=1}^T\mathbb{E}\left[C_t\right] \\
    &=\sum_{t=1}^T\mathbb{E}\left[C_t\mid \x_t=\x \right]p(\x_t=\x) + \mathbb{E}\left[C_t\mid \x_t\neq\x \right]p(\x_t\neq\x) \\
    &=\sum_{i=1}^Tp_\lambda(C_t=1\mid\x_t=\x)p(\x_t=\x) + p_\lambda(C_t=1\mid\x_t\neq \x)p(\x_t\neq\x),
\end{align}
which is a mixture of two functions that are linear with respect to $\lambda$. As the mixture parameters are independent of $\lambda$, the total expected value is linear.

\clearpage
\section{Experimental Details}\label{app:exp_details}
For all experiments we use a linear scheduler for $\gamma_t$, i.e., $\gamma_t=1-t$. In the following, we describe dataset specific details.
\subsection{Datasets}\label{app:datasets} 
\input{tables/app_datasets}
\subsubsection{Text Datasets}
We summarize the details of the text dataset in \cref{tab:datasets}.
\paragraph{LM1B.} 
We download the dataset from \href{https://huggingface.co/datasets/billion-word-benchmark/lm1b}{Hugging Face Datasets (LM1B)}, and use the bert-base-uncased tokenizer. We pad the sequences to a maximum length of 128. We use the official training, validation, and test splits provided in the dataset.

\paragraph{OpenWebText.} We tokenize using the GPT-2 tokenizer and wrap the sequences to a maximum length of 1024. We adopt the same train-validation splits from \citet{sahoo2024simple}, where the last 100,000 samples are used for validation.
\subsubsection{Molecule datasets}
\paragraph{QM9.}
The QM9 dataset~\citep{wu2018moleculenet} consists of molecules with up to 9 heavy atoms. We use the same split established by \citep{vignac2022digress}, and 10000 graphs for evaluation. 
\paragraph{GuacaMol.}
Furthermore, we evaluate our model on the GuacaMol dataset~\citep{brown2019guacamol}. The molecules range from 2 to 88 heavy atoms.
\paragraph{MOSES.}
Finally, we benchmark our model on the MOSES dataset~\citep{polykovskiy2020molecular}, spanning filtered molecules from 8 to 27 heavy atoms.
\subsection{Model architectures}\label{app:model_architectures}
\input{tables/app_architectures}
\paragraph{Text datasets.} For LM1B and OWT training, we use the modified Transformer \citep{peebles2023scalable} architecture from \citet{sahoo2024simple}. The model consists of 12 layers, a hidden dimension of 768, and 12 attention heads.
\paragraph{Graph datasets.}
For the graph datasets, we use the graph transformer proposed by~\citep{vignac2022digress} with the hyperparameters used by~\citep{qin2024defog} and Relative Random Walk Probabilities~\citep{ma2023graph} as node and edge features. We show our hyperparameters in \cref{tab:graph_hyperparameters}.
\input{tables/graph_hyperparameters}
\subsection{Training settings}\label{app:training_settings}
\input{tables/app_training_settings}
In \cref{tab:graph_hyperparameters} and \cref{tab:training_settings}, we report the training hyperparameters used in our experiments for the different datasets.

\subsection{Evaluation metrics}\label{app:eval_metrics}

\subsubsection{Text Generation}
\paragraph{Perplexity.} Perplexity is computed as the exponential of the negative average token-level log-likelihood over the evaluation data. We report the perplexities on the test splits following \citet{sahoo2024simple}.
\paragraph{Generative perplexity.} We report generative perplexity for unconditionally generated sequences, following \citep{sahoo2025diffusion}. We measure it using GPT-2 large \citep{radford2019language}. We further report sample entropy to assess diversity.

\subsubsection{Graph Generation}
We evaluate the graph generation with standard molecule metrics. We report the validity, uniqueness, i.e., proportion of generated molecules with different SMILES string, and novelty, i.e., proportion of generated molecules not in the training set. Furthermore, we report the Frechet ChemNetDistance (FCD)~\citep{preuer2018frechet}, which measures the embedding similarity of the generated molecules with respect to the training set. Finally, we report the Scaffold similarity and SNN.

\clearpage
\section{Additional Results}\label{app:additional_results}
In the following, we show hyperparameter ablations for $\rho$ and $\lambda$ across the datasets QM9, Moses, GuacaMol, and LM1B in ~\cref{fig:qm9_heatmap,fig:moses_heatmap,fig:guacamol_heatmap,fig:ppl_vs_steps}.
\input{includes/heatmap_qm9}
\input{includes/heatmaps_moses}
\input{includes/heatmap_guacamol}
\input{includes/ppl_vs_steps}

%% file: tables/app_datasets.tex
\begin{table}[h]
\centering
\caption{Additional details on the datasets.}
\label{tab:datasets}
\begin{tabular}{lccc}
\hline
Dataset & Tokenizer & Vocab. Size & Sequence Length \\
\hline
OWT & GPT2 & 50,257 & 1024  \\
LM1B & bert-base-uncased & 30,522 & 128 \\
\hline
\end{tabular}
\end{table}

%% file: tables/app_architectures.tex
\begin{table}[h]
\centering
\caption{Model architectures on various datasets.}
\label{tab:architectures}
\begin{tabular}{lcc}
\hline
Dataset & Architecture & Parameter Count \\
\hline
LM1B & Transformer & 139 M \\
OWT & Transformer & 169 M \\
\hline
\end{tabular}
\end{table}

%% file: tables/graph_hyperparameters.tex
\begin{table}[h]
\centering
\caption{Overview of the training hyperparameters for the different experiments.}
\label{tab:graph_hyperparameters}
\begin{tabular}{lccc}
\hline
 & QM9 & Guacamol & Moses \\
\hline
Epochs        & 1000 & 250 &  300\\
%Batch size         & 512 & 64\\
Learning rate      & $2\times10^{-4}$ & $2\times10^{-4}$ & $2\times10^{-4}$ \\
Optimizer          & ADAM $(0.9, 0.999)$ & ADAM $(0.9, 0.999)$ & ADAM $(0.9, 0.999)$ \\
RRWP steps   & 12 & 12 & 20 \\
Number of layers & 9 & 10 & 12 \\
Number of heads & 8 & 8 & 8 \\
$\rho$ & 4 & 4 & 4 \\
$\lambda$ & 0.0 & 0.2 & 0.4\\
\hline
\end{tabular}
\end{table}

%% file: tables/app_training_settings.tex
\begin{table}[h]
\centering
\caption{Overview of the training hyperparameters for the different experiments.}
\label{tab:training_settings}
\begin{tabular}{lcc}
\hline
 & LM1B & OWT \\
\hline
Train steps        & 1000K & 150K\\
Context size       & 128 & 1024 \\
Batch size         & 512 & 64\\
Learning rate      & $1\times10^{-3}$ & $3\times10^{-4}$ \\
Optimizer          & ADAM $(0.9, 0.999)$ & ADAM $(0.9, 0.999)$ \\
LR warmup steps    & 2.5K & 2.5K \\
\hline
\end{tabular}
\end{table}

%% file: includes/heatmap_qm9.tex
\begin{figure*}[h]
\centering
\begin{subfigure}{0.5\textwidth}
  \centering
  \includegraphics[width=\linewidth]{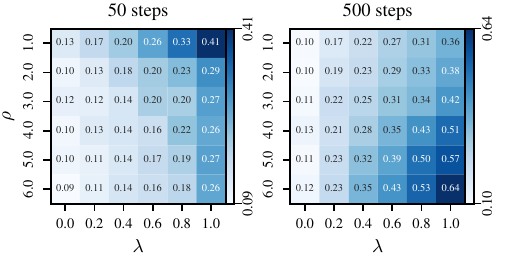}
    \vspace{-1em}
  \caption{FCD $\downarrow$}
\end{subfigure}%
\begin{subfigure}{0.5\textwidth}
  \centering
  \includegraphics[width=\linewidth]{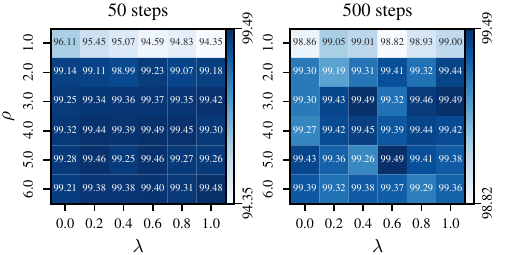}
    \vspace{-1em}
  \caption{Validity $\uparrow$}
\end{subfigure}
\vspace{-1em}
\caption{FCD and validity for QM9 across different parameters.}
\label{fig:qm9_heatmap}
\end{figure*}

%% file: includes/heatmaps_moses.tex
\begin{figure*}[h]
\centering
\begin{subfigure}{0.5\textwidth}
  \centering
  \includegraphics[width=\linewidth]{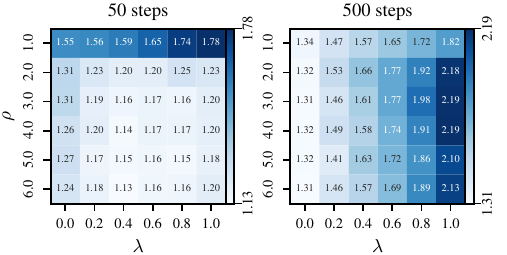}
  \vspace{-1em}
  \caption{FCD $\downarrow$}
\end{subfigure}%
\begin{subfigure}{0.5\textwidth}
  \centering
  \includegraphics[width=\linewidth]{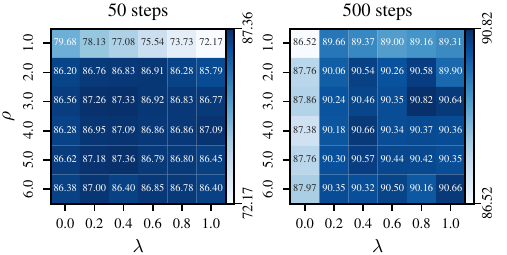}
  \vspace{-1em}
  \caption{Validity $\uparrow$}
\end{subfigure}
\vspace{0.75em}
\begin{subfigure}{0.5\textwidth}
  \centering
  \includegraphics[width=\linewidth]{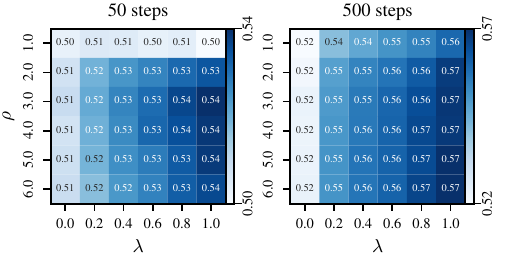}
    \vspace{-1em}
  \caption{SNN $\uparrow$}
\end{subfigure}%
\begin{subfigure}{0.5\textwidth}
  \centering
  \includegraphics[width=\linewidth]{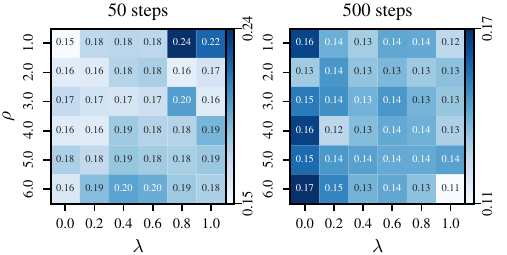}
    \vspace{-1em}
  \caption{Scaf $\uparrow$}
\end{subfigure}
  \vspace{-1em}
\caption{FCD, validity, SNN, and Scaffold similarity for MOSES across different parameters.}
\label{fig:moses_heatmap}
\end{figure*}

%% file: includes/heatmap_guacamol.tex
\begin{figure*}[h]
\centering
\begin{subfigure}{0.5\textwidth}
  \centering
  \includegraphics[width=\linewidth]{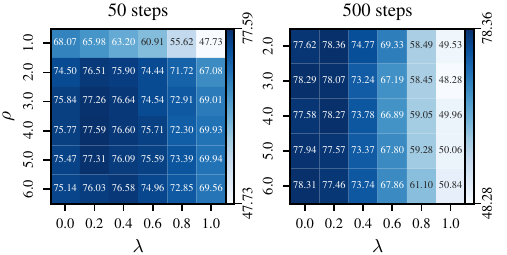}
    \vspace{-1em}
  \caption{FCD $\downarrow$}
\end{subfigure}%
\begin{subfigure}{0.5\textwidth}
  \centering
  \includegraphics[width=\linewidth]{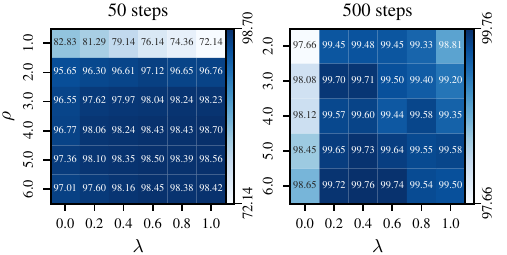}
    \vspace{-1em}
  \caption{Validity $\uparrow$}
\end{subfigure}
  \vspace{-1em}
\caption{FCD and validity for Guacamol across different parameters.}
\label{fig:guacamol_heatmap}
\end{figure*}

%% file: includes/ppl_vs_steps.tex
\begin{figure}[htpb]
    \centering
    \begin{subfigure}[t]{0.48\linewidth}
        \centering
        \includegraphics[width=\linewidth]{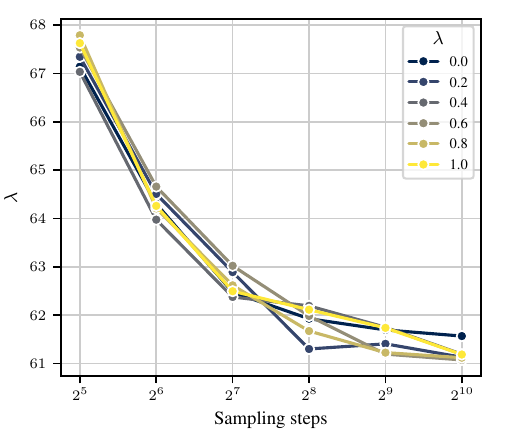}
        \caption{Perplexity vs. sampling steps}
        \label{fig:ppl_vs_steps_line}
    \end{subfigure}
    \hfill
    \begin{subfigure}[t]{0.48\linewidth}
        \centering
        \includegraphics[width=\linewidth]{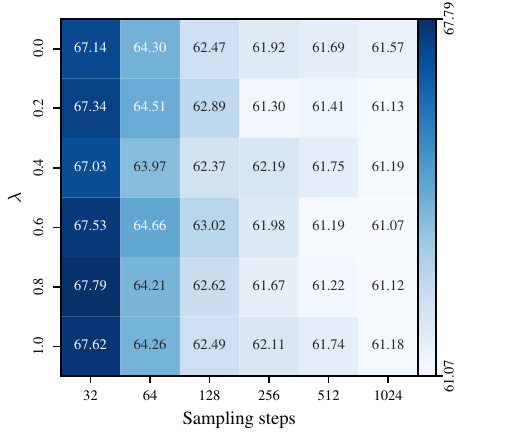}
        \caption{Heatmap}
        \label{fig:ppl_vs_steps_heatmap}
    \end{subfigure}

    \caption{Perplexity on LM1B as a function of the number of sampling steps for different mixing scheduler values $\lambda$.}
    \label{fig:ppl_vs_steps}
\end{figure}